\documentclass{article}
\usepackage{times}
\usepackage{color}
\usepackage{graphics}
\usepackage{amssymb}
\usepackage{amsmath}
\usepackage[dvips]{graphicx}
\usepackage{float}
\floatstyle{ruled}
\usepackage{framed}
\usepackage{subfigure}
\usepackage{tikz}
\usepackage{bbold}
\usepackage{pifont}
\usepackage{bm}
\usepackage{amssymb,amsmath}
\usetikzlibrary{arrows,shapes,automata,petri,positioning}
\usepackage{amsthm}
\usepackage{amsthm}
\newtheorem{theorem}{Theorem}
\usepackage{hyperref}

\usepackage{graphicx} % more modern
\usepackage{subfigure} 

\usepackage{natbib}

\usepackage{algorithm}
\usepackage{algorithmic}

\usepackage{hyperref}

\usepackage[accepted]{icml2013}
\icmltitlerunning{Adaptive Variational Particle Filtering in Non-stationary Environments}

\newcommand{\commentedbox}[2]{%
  \mbox{
    \begin{tabular}[t]{@{}c@{}}
    $\boxed{\displaystyle#1}$\\
    #2
    \end{tabular}%
  }%
}

\tikzset{
    place/.style={
        circle,
        thick,
        draw=blue!75,
        fill=blue!20,
        minimum size=6mm,
    },
    transitionH/.style={
        rectangle,
        thick,
        fill=black,
        minimum width=8mm,
        inner ysep=2pt
    },
    transitionV/.style={
        rectangle,
        thin,
        fill=blue,
        minimum height=8mm,
        inner xsep=2pt
    }
}

\begin{document} 

\twocolumn[
\icmltitle{Adaptive Variational Particle Filtering in Non-stationary Environments}

% unless you've provided the [accepted] option to the icml2013
\icmlauthor{Mahdi Azarafrooz}{mazarafrooz@cylance.com}
\icmladdress{Cylance Inc., Irvine, CA USA}
\icmlkeywords{boring formatting information, machine learning, ICML}

\vskip 0.3in
]

\begin{abstract} 
Online convex optimization is a sequential prediction framework with the goal to track and adapt to the environment through evaluating proper convex loss functions. 
We study efficient particle filtering methods from the perspective of such framework. 

We formulate an efficient particle filtering methods for non-stationary environment by making connections with the online mirror descent algorithm which is known to be universal online convex optimization algorithm. 
 As a result of this connection, our proposed particle filtering algorithm proves to achieve optimal particle efficiency. 
\end{abstract} 

\section{Introduction}
%Bayesian methodology benefits from prediction accuracy due to rich representational power for uncertainty.  It is free of the overfitting concerns that come with deep learning and other non-Bayesian approaches. The conservative nature of the Bayesian approach avoids overfitting by taking the average over the space of uncertainties. The major downside of Bayesian modeling is the computational cost for marginalizing out the uncertainties in the inference procedures.  
Inference in {\it online} settings is challenging. It is because not only the posterior distributions should be approximated sequentially, but also one should take into account the {\it non-stationary} characteristics of the data. In these situations, suitable sequential inference procedures are required to {\it track} and {\it bind} the suitable probabilistic distance between the approximate and true posterior distribution at each time step. Otherwise the prediction errors would be propagated and leads to very poor predictions. A suitable framework for achieving online inference is online convex optimization (OCP) framework [18]. It is a sequential prediction framework with the goal to track and adapt to the environment through evaluating proper convex loss functions. By making connection between OCP and a recent efficient particle inference, we propose a particle filtering method that is suitable for online inference algorithms.
  
\textbf{Main related works} Using online convex optimization for sequential posterior approximation has been already discussed in [16,22]. However they parametrize the environment strategy space through the use of parameterized exponential family of distributions which resembles the variational inferences methodology and therefore carrying its limitations. Moreover the calculation of the loss term in [22] is not straightforward and requires other inference algorithms such as Monte Carlo Marko Chains (MCMC). Our proposed algorithm computes the loss term by reusing the generated particle.  Another interesting related work to ours is [19]. It takes advantage of the mirror descent algorithm, particle methods and provides theoretical guarantees. However, 3 main key differences are:  a) we address online filtering for time series data. b) {\it unlike [19,4], our proposed online inference algorithm doesn't have any assumption on the length of data a priori}. c) Unlike [19] our proposed method follows a deterministic particle selection approach. [23]  builds an online inference algorithm using particle learning (PL) method [24]. PL methods differ from sequential Monte Carlo (SMC) in that PL reverses the order of resampling and propagation procedures. 
%Moreover non of the previous approaches address inference in non-stationary environments. 
% In the case of artificially generated data, they assert that their performance is comparable to SLICE sampling [25] and therefore better than GIBBS sampling [13].
 
\section{Preliminaries}

 All the vectors are denoted by bold symbols.
 
 \subsection{Online convex optimization}\label{cop}
%Online convex optimization (OCP) is a {\it sequential} prediction framework with the goal to track and adapt to the environment through evaluating proper convex loss functions.  
One way to explain the OCP framework, in the context of {\it adversarial} environments, is by considering it as a repeated game between a forecaster and an adversary. It includes the following key ingredients:
\begin{enumerate}
\item Parameters of the environment $\bold{w}_n \in \Delta$, forecaster strategies $Q(\bold{w}_n) \in \mathcal{Q}$ , adversary strategies  $S(\bold{w}_n) \in \mathcal{S}$. %\footnote{Selecting previously mentioned notations is intentional.}
 \item convex loss function $l(.)$ over the strategies.
\item {\it Shifting/Tracking Regret} defined as the following minimax metric:
%\begin{eqnarray}\label{regret_function_shifting}
%\begin{array}{l}
%\mathcal{R}_N^{shifting}\equiv \\
%\begin{split} 
 %\displaystyle \min_{Q\in \mathcal{Q}} \max_{[r,s] \in [1,N]} \max_{\bold{w}\in \Delta} \max_{S \in \mathcal{S}} \sum_{n=r}^{s} l(Q(\bold{w}_n))-l(S(\bold{w}_n))
%\end{split}
%\end{array}
%\end{eqnarray}
%\end{enumerate}
\begin{eqnarray}\label{regret_function_shifting}
\begin{array}{l}
\mathcal{R}_N^{shifting}\equiv \\
\begin{split} 
 \displaystyle \min_{Q\in \mathcal{Q}} \max_{\bold{w}\in \Delta} \max_{S \in \mathcal{S}} \sum_{n=1}^{N} l(Q(\bold{w}_n))-l(S(\bold{w}_n))
\end{split}
\end{array}
\end{eqnarray}
\end{enumerate}

where $N$ is the length of the sequence of observations. The forecaster's goal is to yield the lowest total loss when played against the {\it arbitrary} sequence of adversary strategies. The goal of the OCP algorithms is to achieve sub-linearly bounded regret metric. One universal OCP algorithm is the online mirror descent algorithm MD [14].  MD casts the OCP framework to a first-order loss decrease optimization problem with Bregman distance regularizer as in Eq. \ref{class-MD}:
 \begin{eqnarray}\label{class-MD}
\begin{array}{l}
\begin{split} 
\hat{\bold{w}}_{n+1}=\underset{w_n \in \Delta}{\text{argmin}} \hspace{.1cm} \epsilon_n \langle \bold{w}_n,\nabla_{\bold{w}_n}l_n \rangle  + D(\bold{w}_n\|\hat{\bold{w}}_n)
\end{split}
\end{array}
\end{eqnarray}
where $\nabla_{\bold{w}}$ is the gradient w.r.t parameters of the environment , $\epsilon_n$ is  the step size parameter and $D$ is the Bregman distance. Given the fixed per-observation computational complexity constraint, the MD algorithm [16] achieves the optimal shifting regret of $\sqrt N$.  In the section \ref{MD}, we show how we can incorporate an online MD algorithm into particle filtering algorithms.
%Note that the prediction strategy $Q$ is observation oriented, but the environment strategy can be arbitrary and dynamic. 

\subsection{Particle Filtering}

Assume a hidden Markov model (HMM) with the observations $\bold{y}_{N}=\{y_1,...,y_N\}$, the hidden states $\bold{x}_{N}=\{x_1,...,x_N\}$, $x_n\in \{1,...,M_n\}$ and $\boldsymbol{\theta}_{N}=\{\theta_1,...,\theta_N\}$ as a set of parameters that controls the transition and emission processes in Eq. \ref{state_space}: 

\begin{equation}\label{state_space}
\begin{array}{l}
p(x_{n},\theta_{n},y_n|x_{n-1},\theta_{n-1})=\\ \displaystyle p(\theta_{n}) \commentedbox{p(y_n|x_n,\theta_{n})}{\tiny emission}  \commentedbox{p(x_n|x_{n-1},\theta_{n-1})}{\tiny transition}
\end{array}
\end{equation}
%\subsection{Sufficient statistics as the magic of non-parametric models}
We build our framework based on flexible non-parametric models. The complexity of these models (for example the number of hidden states, etc.) increase as the amount of data grows, in a flexible manner. To this aim, an infinite capacity HMM (iHMM) model assumes that the posterior distribution of parameters $p(\boldsymbol{\theta}_{n})$ depends on the hidden states $\bold{x}_{n}$ and observations $\bold{y}_{n}$ through a low dimensional vector of {\it sufficient statistics} $c_n$  
%\footnote{The existence of such sufficient information is guaranteed in non-parametric Bayesian (infinite capacity models) due to their tie to exchangeability characteristic and de Finetti's representation theorem},
i.e $p(\boldsymbol{\theta}_{n}|\bold{x}_{n},\bold{y}_{n})=p(\boldsymbol{\theta}_{n}|c_n)$. The sufficient statistic must be updated using a {\it deterministic} recursion algorithm $\mathcal{C}$ sequentially such that $c_{n}=\mathcal{C}(c_{n-1},x_{n-1},y_{n-1})$. The existence of such deterministic recursion as well as proper analytical integrations imply that one can marginalize out the parameters $\boldsymbol{\theta}_{N}$ from Eq. \ref{state_space} as in  Eq. \ref{marginalized_HMM}.
\begin{eqnarray}\label{marginalized_HMM}
\begin{array}{l}
p(x_{n},y_n|x_{n-1},c_{n-1})= \\  \commentedbox{p(y_n|x_n,c_n=\mathcal{C}(c_{n-1},x_{n-1},y_{n-1}))}{\tiny dependency between observations $y_n$ and $y_{n-1}$ through recursion algorithm $\mathcal{C}(.)$ } \\ \times  p(x_n|x_{n-1},c_{n-1})\end{array}
\end{eqnarray}
\footnote{For better readability, the details on the analytical integration, sufficient statistics $c_n$ and their updating process for iHMM is moved to Appendix I.}  
%Eq. \ref{marginalized_HMM} implies that posterior computations implicitly rely on the sufficient statistics updates and the analytical integrations.  This introduces dependencies between observations as is pointed out in Eq. \ref{marginalized_HMM}.  The focus of our work is on the design of adaptive particle filtering methods where such dependencies would be taken into account.  We accomplish this by borrowing the concept of regret from online convex optimization literature and applying it in the context of filtering for iHMM. 
 %\subsection{Particle Filtering}\label{pf}
In the filtering problems,  one is interested in deriving the posterior $p(\bold{x}_{n+1},\bold{y}_{n+1})$ from $p(\bold{x}_{n},\bold{y}_{n})$.  One popular filtering approach is particle method. This especially comes in handy when the closed form expressions for the posteriors are not available. A well-known example is the sequential Monte Carlo (SMC) method [2]. In the case of generative process in Eq. \ref{marginalized_HMM}, it first initializes a set of particles $\{x_0^k\}_{k=1}^K$. Also for the readability of the paper, we assume that each particle encapsulates the sufficient update process $\{c_{n}^k=\mathcal{C}(c_{n-1}^k,x_{n-1}^k,y_{n-1})\}_{k=1}^{K}$ as well as the analytical integration. Therefore from now on  $x_n^k$ represents $(x_n^k,c_n^k)$.

 It then follows these recursive steps:
\begin{enumerate}
\item {\it Propagation.} Using particles $\{x_n^k\}_{k=1}^K$:
\begin{eqnarray}\label{SMC}
\begin{array}{l}
\{W_n^k=p(y_n,x_n^k)\}_{k=1}^{K} , w_n^k=\frac{W_n^k}{\sum_{j=1}^KW_n^j} \\
\end{array}
\end{eqnarray}
\begin{eqnarray}\label{MC}
\begin{array}{l}
 Q(\bold{x}_{n},\bold{y}_{n})=\displaystyle \sum_{k=1}^{K}w_n^k\delta[x_n,x_n^k] 
 \end{array}
\end{eqnarray}
where $\delta$ is Kronecker-delta.%\footnote{$\delta(a,b)=1$ iff $a=b$ and $0$ otherwise}.
\item {\it Resampling}. Sample  $K$ new particles with replacement $\{x_{n+1}^k\}_{k=1}^K   \sim \text{Multinomial($K,w_n^1,...,w_n^K$)}$
\end{enumerate}

%The nice property of SMC is that convergence in probability is guaranteed as $K \rightarrow \infty$, i.e $  p(\bold{x}_{n},\bold{y}_{n})  \xrightarrow{D}  Q(\bold{x}_{n},\bold{y}_{n})$ where $Q$ is the approximated posterior.
%Its disadvantage is that its resampling step makes the particles concentrated on a small subset of the sample space, known as the degeneracy issue. Therefore it requires significant number of particles  to achieve the desirable approximation.  Avoiding the resampling step on the other hand, sets free the low probability particles which causes high variance estimations. 

\section{Mirror descent variational particle filtering}\label{MD}
In order to reformulate the online filtering problem from an OCP perspective, we treat the environment as a posterior distribution, the strategies as the filtering strategies $Q(\bold{w}_n) $ and the MC particles and their weights $\bold{w}_n =[w_n^1,...,w_n^{K\times M_n}]   \in \Delta^{K\times M_n} $ as the parameters of the environment where $M_n$ is the number of hidden states at sequence index $n$ and $K$ is the number of particles. Moreover the constraint of fixed per-observation computational complexity in OCP gets translated to the constraint of fixed $K$ number of particles.

Similar to [3]  we builds our framework based on the Markov random field (MRF) assumption for the posterior $p(\bold{x}_{n},\bold{y}_{n})$ as in Eq. \ref{Markov_network} 

\begin{eqnarray}\label{Markov_network}
\begin{array}{l}
\displaystyle p(\bold{x}_{n},\bold{y}_{n})= \frac{ \prod_{i=1}^{n} f(x_{i},y_{i}) }{\exp(\Phi)}
\end{array}
\end{eqnarray}
where $\Phi$ is the log-partition function and $f(x_{i},y_{i})= p(y_i|x_i)  p(x_i|x_{i-1})$  the potential function. %The term $\Phi$ can be written as the summation of two other terms by introducing an auxiliary distribution $\mathcal{Q}(\bold{x}_{n},\bold{y}_{n})$ as in Eq. \ref{VI-relation}: 

\begin{algorithm}
\caption{Mirror Descent Variational Particle Filtering (MD-VPA)}
\label{MD-VPA}
\begin{algorithmic}[1]
\STATE {\bfseries Input:} A decreasing sequence of strictly positive discounting factors $\{\epsilon_n\}$, e.g $\frac{1}{n}$, $M_0$, $K$ and initial weight vectors $\bold{w}_0=[w_0^1,...,w_0^{K\times M_0}]  \in \Delta^{K\times M_0}$
\FOR{$n=1,2,3,...,N$}
\FOR{$m=1,...,M_n ~ \& ~ k=1,...,K$}
\STATE \begin{eqnarray}\label{md-update}
\begin{split} 
W_n(k,m)=\\  \commentedbox{f(x_n^k=m,y_n)}{\tiny Free energy related term}  \commentedbox{\frac{ f(x_n^k=m,y_{n+1})^{\epsilon_n}} {( \sum_{k=1}^{K}f(x_n^k=m,y_{n+1}))^{\frac{\epsilon_n}{K}}}}{\tiny regret related term} \nonumber
\end{split}
%\end{array}
\end{eqnarray} 
\ENDFOR
\STATE Select $K$ largest $W_n$ and normalize the weights $w_n^k$ for the new particles $x_n^k$.
\STATE  $Q(\bold{x}_{n},\bold{y}_{n})=\displaystyle \sum_{k=1}^{K}w_n^k\delta[x_n,x_n^k]$
\ENDFOR
\end{algorithmic}
\end{algorithm}

Let's define the loss function as the average of the predictive log-likelihood over the space of the environment space parameters $\bold{w}_n \in \Delta_n^{K \times M_n}$ as in Eq. \ref{loss}:
\begin{eqnarray}\label{loss}
\begin{array}{l}
\begin{split} \small
l_{n}(\bold{w}_n)=\\ -\int_{\bold{w}_n \in \Delta_n^{K \times M_n}} \log(p(\bold{x}_{n},\bold{y}_{n}^{y_n \leftarrow y_{n+1}})) d\bold{w}_n
\end{split}
\end{array}
\end{eqnarray}

%with 
%\begin{eqnarray}\label{Markov_network_parametrized}
%\begin{array}{l}
%\displaystyle p(\bold{x}_{n},\bold{y}_{n}^{y_n \leftarrow y_{n+1}})= \frac{ f(x_{n},y_{n+1})\prod_{i=1}^{n-1} f(x_{i},y_{i}) }{\exp(\Phi(\bold{w}_n))}
%\end{array}
%\end{eqnarray}

It is as if we were observing $y_{n+1}$ instead of $y_{n}$ and calculate the potential function $f_n$ with $y_{n+1}$, given that nothing else in the history of hidden states and observation has changed. 
By replacing Eq. \ref{loss} in Eq. 2 and resolving a coordinate ascent algorithm leads to the proposed Mirror Descent Variational Particle Filtering (MD-VPA) described in Alg. \ref{MD-VPA}. It deletes or propagates $K$ particles {\it deterministically} (out of $K \times M_n$ candidates living in the simplex space $\Delta^{K\times M_n}$) based on their relative contribution to the negative free energy and loss. The derivation of this coordinate-ascent algorithm is demonstrated in Appendix II.
As the result of the MD prediction strategy, the following theorem is in order:
\begin{theorem}
MD-VPA achieves the optimal shifting regret of $\mathcal{O}(\sqrt N)$ for a sequence of length $N$.
\end{theorem}
\begin{proof} 
The proof is almost the same as [16]. However [16] assumes that the set of parameters $\bold{w}_n$ are selected such that strong Hessian convexity of $\Phi(\bold{w}_n) $ holds true. MD-VPA is free of this assumption. Given the finite number of particles $K$, the {\it deterministic} nature of the particle selection process (coordinate-ascent) in MD-VPA dictates all the particles to have high probability. By noting the equivalence of the $\nabla_{\bold{w}_n}^2\Phi(\bold{w}_n)$ with Fisher information matrix, $\mathcal{J}(\bold{w}_n)=-\mathbb E_{\bold{w}_n}(\nabla^2_{\bold{w}_n} \log p(\bold{x}_n,\bold{y}_{n}))$ then it is guaranteed that $\exists H>0,  \nabla^2\Phi(\bold{w}_n)  \succeq H\mathbb{1}_{\small KM_n\times KM_n}$.  In other words, the weight of the particles are sufficiently ``informative'' and therefore strong Hessian convexity requirement is guaranteed.
\end{proof}

\subsection{Discussion on Loss-function}
Using Eq. \ref{loss} one can interpret the loss as the average predictive log-likelihood over the weight of the particles $\bold{w}_n$. Theorem 1 then guarantees that average predictive log-likelihood stays bounded under the well-defined worst case scenarios. It achieves this by casting the variational distance as its regularizer. \footnote{Variational distance is equivalent to Bregman distance for Markov Random Fields } The variational distance however is evaluated using the $y_{n}$ rather than $y_{n+1}$.  We can evaluate it using $y_{n+1}$ instead. But one gets best filtering result when evaluating using $y_{n}$, since this way we can incorporate the information of 2 consecutive observations in a one-pass filtering fashion. This is where the non-parametric nature of the iHMM comes to the picture where evaluating the loss (evaluated using $y_{n+1}$) and variational free energy (evaluated using $y_n$) in one single step is a straightforward task thanks to the existence of sufficient statistics. 
The idea of incorporating the loss term and MD algorithm into Bayesian Inference is mentioned in [22] as well. Their method however requires computing the loss term using the MCMC approach.  The natural {\it particle} interpretation of the Alg. \ref{MD-VPA} enables it to compute the term $\mathbb E_{\bold{w}_n}(\log f(x_n,y_{n+1}))= \sum_{k=1}^{K}\frac{\log f(x^k_n,y_{n+1})}{K}$  by reusing the generated particle. Moreover, the sufficient statistics $c_n$ and the deterministic update algorithm $\mathcal{C}$, $\{c_{n}^k=\mathcal{C}(c_{n-1}^k,x_{n-1}^k,y_{n-1})\}_{k=1}^{K}$ comes handy to compute the $\{f(x^k_n,y_{n+1})\}_{k=1}^{K}$ terms easily. This makes the implementation of the MD-VPA algorithm very straightforward.

\section{Simulation Results}\label{sim}

We ran our simulations for both artificial and real data with the hyper-parameters $\alpha=\gamma=1$.  $\epsilon_n=\frac{1}{n}$ is set for all MD-VPA filtering algorithms. The hyper-parameters are explained in the Appendix I.
% Also simulations are conducted for equal number of particles, regardless of the type of the filtering problem. This is because we are comparing the particle efficiency of the particle filtering methods. 
%A particle filtering method is more efficient than another one, when it achieves higher log-likelihood and lower variance given fixed and similar number of particles. 

\subsection{Artificial Non-Stationary  Data}
We first generated a sequence of 150 data using a (non-negatively correlated) HMM with 3 states and the following transition and emission matrixes respectively:

$\tiny
 \begin{bmatrix}
   0 & 1/2 &1/2 \\
    1/2 & 1/2 & 0 \\
    1/2 & 0 & 1/2
\end{bmatrix}
,
 \begin{bmatrix}
   1/2 & 0 & 1/2 \\
    1/3 & 1/3 & 1/3 \\
     0 & 1/2 & 1/2 
\end{bmatrix}
$
concatanated by a negatively correlated HMM with 4 states and a multinomial emission distribution with 8 categories using the following transition and emission matrixes respectively:
 \[\tiny
 \begin{bmatrix}
   0 & 1/2 &1/2 & 0\\
    0 & 0 &1/2 & 1/2 \\
   1/2 & 0 & 0 & 1/2 \\
    1/2 & 1/2 &0 & 0
\end{bmatrix}
\]
 \[\tiny
 \begin{bmatrix}
   1/3 & 0 & 0 & 0 & 0 & 0 & 1/3 &1/3\\
    1/3 & 1/3 & 1/3 & 0 & 0 & 0 & 0 &0 \\
    0 & 0 & 1/3 & 1/3 & 1/3 & 0 & 0 &0 \\
    0 & 0 & 0 & 0 & 1/3 & 1/3 & 1/3 &0
\end{bmatrix}
\]

 Results for  $K=100$ particles are reported in Fig. \ref{changing-loglikelihood-error} . We purposely generated the first half of the sequence (first 150 sequence) using a non-negatively correlated HMM with 3 states, to show that SMC performs better when less exploration in the state space is required. However as posterior gets updated, MD-VPA tracks faster and performs better than SMC and VPA in terms of both the predictive log-likelihood and the estimation variance.

\begin{figure}[h!]
 \caption{Particle filtering for changing posterior}\label{changing-loglikelihood-error}
   \centering
    \includegraphics[width=0.25\textwidth]{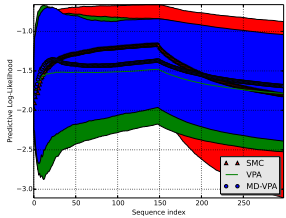}
\end{figure}

\subsection{Alice in Wonderland, Harry Potter and War and Peace}
We concatenated 600 subsequent characters from beginning of ``Alice in Wonderland'', 600 from ``Harry Potter"  and 600 from ``War and Peace".  The results are shown in Fig. \ref{changintext} for 50 particle and 50 random initial states. MDA outperforms SMC and VPA in terms of both the predictive log-likelihood and the estimation variance.

\begin{figure}[h!]
 \caption{Alice in Wonderland, Harry Potter , War and Peace}\label{changintext}
   \centering
    \includegraphics[width=0.25\textwidth]{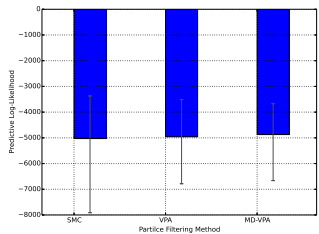}
\end{figure}

\subsection{Web-Click}
MD-VPA performs exceptionally good when it gets applied to the MSNBC.com Web Data Set  [26]. It contains sequential categorical data collected from news-related portions of msn.com. Each sequence in the dataset corresponds to page views of a user. Each event in the sequence corresponds to a user's request for a page. Requests are recorded at the level of page category.  It is natural that different users have different interests for visiting pages. Therefore data contains arbitrary sequences of users' web-hopping strategies. The results are shown in Fig. \ref{WebClick} with 100 particles. We have avoided plotting the estimation error as we observed no considerable difference between the compared algorithms. 

\begin{figure}[h!]
 \caption{Particle filtering for web-click }\label{WebClick}
   \centering
    \includegraphics[width=0.25\textwidth]{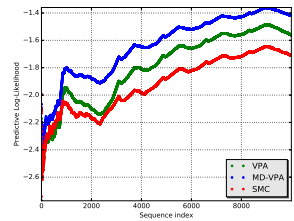}
\end{figure}

\section{Conclusion}
The main novelty of our work is to address the efficient particle inference for non-stationary sequential data from the perspective of online convex optimization approaches.  MD-VPA is implemented for iHMM modeling of the artificially generated data as well as the text and web data. It approximates and tracks the change in the posterior faster and more efficiently compared with other filtering mechanisms.  One interesting future work is to compare MD-VPA against adversarial environments. Strong links between the particle filtering methods and problem of sequential lossless coding can be established using our work and the results in [16,20]. For example, the {\it particle efficiency} concept in the online particle filtering methods can be mapped to the concept of {\it code redundancy} in the sequential lossless coding. An interesting future work can be examining these connections in more details. 

\section{REFERENCES}

\small

[1] Wainwright,M. J. \ \& M. I. Jordan. \ (2008) Graphical models, exponential families, and variational inference. {\it Foundations and Trends in Machine Learning, 1(1-2)}, pp.\ 1--305.

[2] Doucet, A., De Freitas, N., Gordon, N., et al. \ (2001). Sequential Monte Carlo methods in practice. New York: Springer Press.

[3]  Saeedi, A, Kulkarni, T.D, Mansinghka, V \ \& Gershman. \  (2015) S. Variational particle approximations .  arXiv:1402.5715v3

%[4] M. Hoffman, D. M. Blei, and F. Bach. Online learning for latent Dirichlet allocation. In Neural Information Processing Systems, volume 23, pages 856?864, 2010.

[4] Hoffman, M, Blei, D.M , Paisley, J \ \& Wang. C.  \ (2013) Stochastic variational inference. {\it Journal of Machine Learning Research, 14}. pp.\ 1303--1347.

%[5] C.Wang,J.Paisley,andD.M.Blei.OnlinevariationalinferenceforthehierarchicalDirichletprocess.In Artificial Intelligence and Statistics, 2011.

[5]  Broderick, T, Boyd, N, Wibisono, A, Wilson, AC, \ \& Jordan, M.  \ (2013) Streaming variational Bayes. {\it Advances in Neural Information Processing Systems}. 

[6] Honkela, A \ \& Valpola, H. On-line variational Bayesian learning. \  (2003) {it In 4th International Symposium on Independent Component Analysis and Blind Signal Separation}. pp.\ 803--808.

%[6] Z. Ghahramani and H. Attias. Online variational Bayesian learning. In Slides from talk presented at NIPS 2000 Workshop on Online learning, pages 101-109, 2000.

[7] Tank, A, Foti, N \ \& Fox, E. \ (2015)Streaming variational inference for Bayesian nonparametric mixture models. {\it In International Conference on Artificial Intelligence and Statistics}.

[8] Theis, L \  \& Hoffman, M.D. \ (2015) A trust-region method for stochastic variational inference with applications to streaming data. arXiv preprint arXiv:1505.07649.

[9] Ahmed, A, Ho, Q,  Teo, C.H,  Eisenstein, J,  Xing, E.P  \ \&  Smola, A.J. \ (2011) Online inference for the infinite topic-cluster model: Storylines from streaming text. {\it In International Conference on Artificial Intelligence and Statistics}. pp.\ 101--109.

[10] Yao, L , Mimno, D \  \&  McCallum, A. \ (2009) Efficient methods for topic model inference on streaming document collections.  { \it In ACM Conference on Knowledge Discovery and Data Mining}. pp.\ 937--946.

[11] Doucet, A, Godsill, S \ \& Andrieu, C. \ (2000) On sequential MonteCarlo sampling methods for Bayesian filtering. {\it Statistics and Computing, 10(3)}. pp.\ 197--208.

[12]  Gal, Y \ \& Ghahramani, Z.   \ (2014). Pitfalls in the use of Parallel Inference for the Dirichlet Process. {\it Proceedings of the 31st International Conference on Machine Learning}
%[13] Sinead Williamson, Avinava Dubey, and Eric P Xing, Parallel Markov chain Monte Carlo for nonparametric mixture models, Proceedings of the 30th International Conference on Machine Learning, 2013, pp. 98-106.

[13] Teh, Y. W., Jordan, M. I., Beal, M. J., \ \& Blei, D. M. \ (2006). Hierarchical Dirichlet processes. {\it Journal of the american statistical association}.

[14] Srebro, N, Sridharan, K \ \&  Tewari, A. \ (2011) On the Universality of Online Mirror Descent . {\it Advances in Neural Information Processing Systems 24}.

%[15]  D. Spinellis, Reliable identification of bounded-length viruses is NP-complete, in Information Theory, IEEE Transactions on , vol. 49, no.1, pp.280-284, Jan 2003

[15] Matthew J. Beal, Zoubin Ghahramani and Carl Edward Rasmussen, \ (2001). The Infinite Hidden Markov Model,  { \it Advances in Neural Information Processing Systems 14)}.

[16]  Raginsky, M, Willett, R.M, Horn, C, Silva, J \ \& Marcia, R.F \ (2012).  Sequential anomaly detection in the presence of noise and limited feedback. {\it IEEE Transactions onInformation Theory Vol. 58}. pp.\ 5544--5562. 

[17]  Krichevsky R. E. \ \&   Trofimov V. K.\  (1981). The performance of universal encoding. { \it IEEE Trans. Inform. Theory, vol. IT-27, no. 2}. pp.\ 199--207.

[18] 	Cesa-Bianchi, N \ \& Lugosi, G. \ (2006)Prediction, learning, and games. Cambridge University Press.

[19] Bo Dai, Niao He, Hanjun Dai and Le Song  (2016).  Provable Bayesian Inference via Particle Mirror Descent. {\it 19th International Conference on Artificial Intelligence and Statistics}. pp.\ 985?994.

[20] Shamir, G. I., \ \& Merhav, N. \ (1999). Low-complexity sequential lossless coding for piecewise-stationary memoryless sources. {\it Information Theory, IEEE Transactions on, 45(5)}. pp.\ 1498--1519.

[21] Bo Dai, Niao He, Hanjun Dai and Le Song, Provable Bayesian Inference via Particle Mirror Descent,  {\it The 19th International Conference on Artificial Intelligence and Statistics}, 2016.

[22] Guhaniyogi, R., Willett, R. M., \ \& Dunson, D. B. (2013). Approximated Bayesian Inference for Massive Streaming Data { \it Duke Discussion Paper}.

[23] A. Rodriguez,  (2011). Online learning for the infinite hidden Markov model. { \it Communications in Statistics - Simulation and Computation 40 (6)}. pp.\ 879-893.

[24]  Carlos M. Carvalho, Hedibert F. Lopes, Nicholas G. Polson, and Matt A. Taddy.  (2010). Particle learning for general mixtures. {\it Bayesian Anal Vol. 5}. pp.\ 709-740.

[25] Van Gael, J., Saatci, Y., Teh. \ \& Ghahramani , Z.  (2008). Beam sampling for the infinite hidden Markov model. {\it In Proceedings of the 25th International Conference on Machine Learning (ICML)}.

[26] \href{https://archive.ics.uci.edu/ml/datasets/MSNBC.com+Anonymous+Web+Data}{MSNBC-WebData}

\tiny

\section{Appendix I}

The sufficient statistics are $c_n=(M_n,\{t_{jc}\})$ where $M_n$ is the number of distinct hidden states up to the time $n$ and $\{t_{jc}\}$ is the number of transitions between states $j$ and $c$ up to time $n$.

The analytical integrations is according to the Chinese restaurant franchise in [13]. $x_n$ is assigned to state $c$ with probability proportional to $t_{x_{n-1}c}$ or to a state never visited from $x_{n-1}$, ($t_{x_{n-1}c}=0$) with probability proportional to $\alpha$. If an unvisited state is selected, $x_n$ is assigned to state $c$ with probability proportional to $\sum_j t_{jc}$, or a new state (i.e, one never visited from any state, $\sum_j t_{ic}=0)$ with probability proportional to $\gamma$.  The parameters $\alpha, \gamma$ are the hyper parameters for the iHMM.

The sufficient statistic updating process $\mathcal{C}_n$ is then simply the book keeping of the number of counts $\{t_{jc}\}$ and updating them at each time $n$ recursively.

\section{Appendix  II}
The goal is to solve the Eq. \ref{class-MD}. First note that Variational distance is equivalent to Bregman distance for Markov Random Fields. The using the following relation, we instead maximize negative free energy $\mathcal{L}(\mathcal{Q})$.
\begin{eqnarray}\label{VI-relation}
\begin{array}{l}
\Phi=KL[\mathcal{Q}\| p]+\mathcal{L}[\mathcal{Q}]
\end{array}
\end{eqnarray}

where $KL[\mathcal{Q}\|p]=\sum_{\bold{x}_{n} }\mathcal{Q}(\bold{x}_{n},\bold{y}_{n}) \log\frac{\mathcal{Q}(\bold{x}_{n},\bold{y}_{n})}{p(\bold{x}_n,\bold{y}_{n})} $ and
\begin{eqnarray}\label{negative-free-energy}
\begin{array}{l}
\mathcal{L}(\mathcal{Q})=\sum_{\bold{x}_n }\sum_{i=1}^n\mathcal{Q}(\bold{x}_{n},\bold{y}_{n}) \log\frac{f(x_i,y_{i})}{\mathcal{Q}(\bold{x}_{n},\bold{y}_{n})}
\end{array}
\end{eqnarray}

  Using Eq. \ref{MC}, one can parametrize $\mathcal{Q}$ and in turn the negative free energy term  $\mathcal{L}(\mathcal{Q})$ as follows:
\begin{eqnarray}\label{free-energy-particle}
\begin{array}{l}
\begin{split} 
\mathcal{L}[\bold{w}_n]= \sum_{k=1}^{K \times M_n} w_n^k\log\frac{f(x_n^k,y_{n})}{w_n^kV_n^k}
\end{split}
\end{array}
\end{eqnarray}

Moreover we want to use only $K$ particles (fixed per-observation computational complexity). This introduces the constraint $\sum_{k=1}^K w_n^k=1$. With this constraint being added as a Lagrange multiplier $\lambda$ to the Eq. \ref{MD}, and substituting for $\mathcal{L}[\bold{w}_n]$ using Eq. \ref{free-energy-particle}, we end up with the following formulation:
\begin{eqnarray}\label{free-energy-particle-lagrange}
\begin{array}{l}
\begin{split} 
\bold{w}_{n+1}=\underset{\bold{w}_n \in \Delta^{K\times M_n}}{\text{argmax}} \hspace{.1cm} -\epsilon_n \langle \bold{w}_n,\nabla_{\bold{w}_n}l_{n} \rangle + \\ \mathcal{L}[\bold{w}_n]+ \lambda  (\sum_{k=1}^K \bold{w}_n^k-1)
\end{split}
\end{array}
\end{eqnarray}

Noting that derivatives of log-partition function $\nabla_{\bold{w}_n}(\Phi(\bold{w}_n)) =[\mathbb E_{w^1}(\log f),...,\mathbb E_{w^{K\times M_n}}(\log f)]$  and taking derivative w.r.t $w_n^k$ and equating to zero we obtain:

\begin{eqnarray}\label{free-energy-particle-lagrange-derivative}
\begin{array}{l}

\log f(x_n^k,y_n)-\log w_n^k -\log V_n^k +\lambda-1\\ +  \epsilon_n \log f(x_n^k,y_{n+1})-  \epsilon_n \mathbb E_{w_n^k}(\log f(x_n,y_{n+1}))=0 \implies
\\\noindent \nonumber
w_{n+1}^k=  \frac{Z_{w_{n}}^{-1}f(x_n^k,y_n)f(x_n^k,y_{n+1})^{\epsilon_n}}{\exp(\epsilon_n\mathbb E_{w^k_n}(\log f(x_n,y_{n+1})))V^k}
\end{array}
\end{eqnarray},
where 

\begin{eqnarray}\label{lagrange-z}
\begin{array}{l}
\begin{split} 
Z_{w_n}=\exp(\lambda-1)^{-1}=\\ 
\sum_{k=1}^K\frac{f_n(x_n^k,y_n)f(x_n^k,y_{n+1})^{\epsilon_n}}{\exp(\epsilon_n\mathbb E_{w_n^k}(\log f(x_n,y_{n+1})))V_n^k}
\end{split}
\end{array}
\end{eqnarray}

Moreover the Hessian convexity $\nabla_{\bold{w}_n}^2\Phi(\bold{w}_n) \succ 0$ implies the concavity of the maximization problem and therefore existence of a solution can be confirmed.
By computing $\mathbb E_{w_n}(\log f(x_n,y_{n+1}))= \sum_{k=1}^{K}\frac{\log f(x^k_n,y_{n+1})}{K}$ and replacing it in Eq. \ref{free-energy-particle-lagrange-derivative}, we end up with Algorithm \ref{MD-VPA}. 

\end{document}